\documentclass[conference]{IEEEtran}
\IEEEoverridecommandlockouts
\usepackage{cite}
\usepackage{url}
\usepackage{amsthm}
\usepackage{amsmath,amssymb,amsfonts}
\usepackage{algorithmic}
\usepackage{graphicx}
\usepackage{textcomp}
\usepackage{xcolor}
\usepackage{changepage}
\usepackage{comment}
\usepackage{booktabs}

\usepackage[caption=false,font=footnotesize]{subfig}
\usepackage{graphicx}  % usually already loaded

\newtheorem{lemma}{Lemma}

\newtheorem{problem}{Problem}
\newtheorem{remark}{Remark}
\newtheorem{proposition}{Proposition}

\newtheorem{definition}{Definition}

\usepackage[acronym]{glossaries}
\makeglossaries

\def\BibTeX{{\rm B\kern-.05em{\sc i\kern-.025em b}\kern-.08em
    T\kern-.1667em\lower.7ex\hbox{E}\kern-.125emX}}
\begin{document}

\title{Hierarchical Informative Path Planning via Graph Guidance and Trajectory Optimization}

\author{Avraiem Iskandar$^{1}$, Shamak Dutta$^{1}$, Kevin Murrant$^{2}$, 
Yash Vardhan Pant$^{1}$, and Stephen L. Smith$^{1}$%
\thanks{This work is supported by the National Research Council Canada (NRC).}%
\thanks{$^{1}$Department of Electrical and Computer Engineering, University of Waterloo, Waterloo, ON N2L~3G1, Canada (e-mails: \protect\url{{avraiem.iskandar, shamak.dutta, yash.pant, stephen.smith}@uwaterloo.ca}).}%
\thanks{$^{2}$National Research Council Canada (NRC), St. John's, NL, Canada (e-mail: \protect\url{kevin.murrant@nrc-cnrc.gc.ca}).}%
}

%
\begin{comment}

\author{\IEEEauthorblockN{1\textsuperscript{st} Given Name Surname}
\IEEEauthorblockA{\textit{dept. name of organization (of Aff.)} \\
\textit{name of organization (of Aff.)}\\
City, Country \\
email address or ORCID}
\and
\IEEEauthorblockN{2\textsuperscript{nd} Given Name Surname}
\IEEEauthorblockA{\textit{dept. name of organization (of Aff.)} \\
\textit{name of organization (of Aff.)}\\
City, Country \\
email address or ORCID}
\and
\IEEEauthorblockN{3\textsuperscript{rd} Given Name Surname}
\IEEEauthorblockA{\textit{dept. name of organization (of Aff.)} \\
\textit{name of organization (of Aff.)}\\
City, Country \\
email address or ORCID}
\and
\IEEEauthorblockN{4\textsuperscript{th} Given Name Surname}
\IEEEauthorblockA{\textit{dept. name of organization (of Aff.)} \\
\textit{name of organization (of Aff.)}\\
City, Country \\
email address or ORCID}
\and
\IEEEauthorblockN{5\textsuperscript{th} Given Name Surname}
\IEEEauthorblockA{\textit{dept. name of organization (of Aff.)} \\
\textit{name of organization (of Aff.)}\\
City, Country \\
email address or ORCID}
\and
\IEEEauthorblockN{6\textsuperscript{th} Given Name Surname}
\IEEEauthorblockA{\textit{dept. name of organization (of Aff.)} \\
\textit{name of organization (of Aff.)}\\
City, Country \\
email address or ORCID}
}
\end{comment}
\maketitle

\newacronym{ipp}{IPP}{Informative Path Planning}
\newacronym{gp}{GP}{Gaussian Process}
\newacronym{nlp}{NLP}{Nonlinear Program}
\newacronym{cmaes}{CMA-ES}{Covariance Matrix Adaptation Evolution Strategy}
\newacronym{miqp}{MIQP}{Mixed Integer Quadratic Program}
\newacronym{micp}{MICP}{Mixed Integer Convex Program}
\newacronym{tl}{TL}{Temporal Logic}

\begin{abstract}
\label{sec:abstract}
We study informative path planning (IPP) with travel budgets in cluttered environments, where an agent collects measurements of a latent field modeled as a Gaussian process (GP) to reduce uncertainty at target locations. Graph-based solvers provide global guarantees but assume pre-selected measurement locations, while continuous trajectory optimization supports path-based sensing but is computationally intensive and sensitive to initialization in obstacle-dense settings. We propose a hierarchical framework with three stages: (i) graph-based global planning, (ii) segment-wise budget allocation using geometric and kernel bounds, and (iii) spline-based refinement of each segment with hard constraints and obstacle pruning.  By combining global guidance with local refinement, our method achieves lower posterior uncertainty than graph-only and continuous baselines, while running faster than continuous-space solvers (up to $9{\times}$ faster than gradient-based methods and $20{\times}$ faster than black-box optimizers) across synthetic cluttered environments and Arctic datasets.
\end{abstract}

\section{Introduction}
\label{sec:introduction}
Consider a specialized ice-breaker ship navigating the Arctic, equipped with sensors to collect measurements for building an accurate model of environmental properties (e.g., sea ice thickness, concentration, salinity, etc.) at specific locations of interest \cite{nicolaus2022overview}, as illustrated in Fig. 1.  Such an expedition is important for environmental monitoring to better understand the impacts of climate change. Ideally, the ship would visit each target location to gather measurements directly, yielding the most precise model. However, such a route would be excessively long and fuel-intensive. As a result, the vessel must plan its trajectory under a strict travel budget constraint. This trade-off between information gain and resource limitations defines the \gls{ipp} problem. The problem is ubiquitous and arises across a wide range of applications, including ocean monitoring \cite{smith2011persistent, hitz2017adaptive, binney2012branch}, 
air quality and atmospheric pollution monitoring~\cite{velasco2020adaptive, marchant2014bayesian},
terrain mapping \cite{popovic2020informative,quann2018ground},
 soil property estimation \cite{roberts2025kriging},
 and autonomous exploration in space missions~\cite{rockenbauer2024traversing}.

\begin{figure}[t]
    \centering
    \includegraphics[width=0.9\columnwidth]{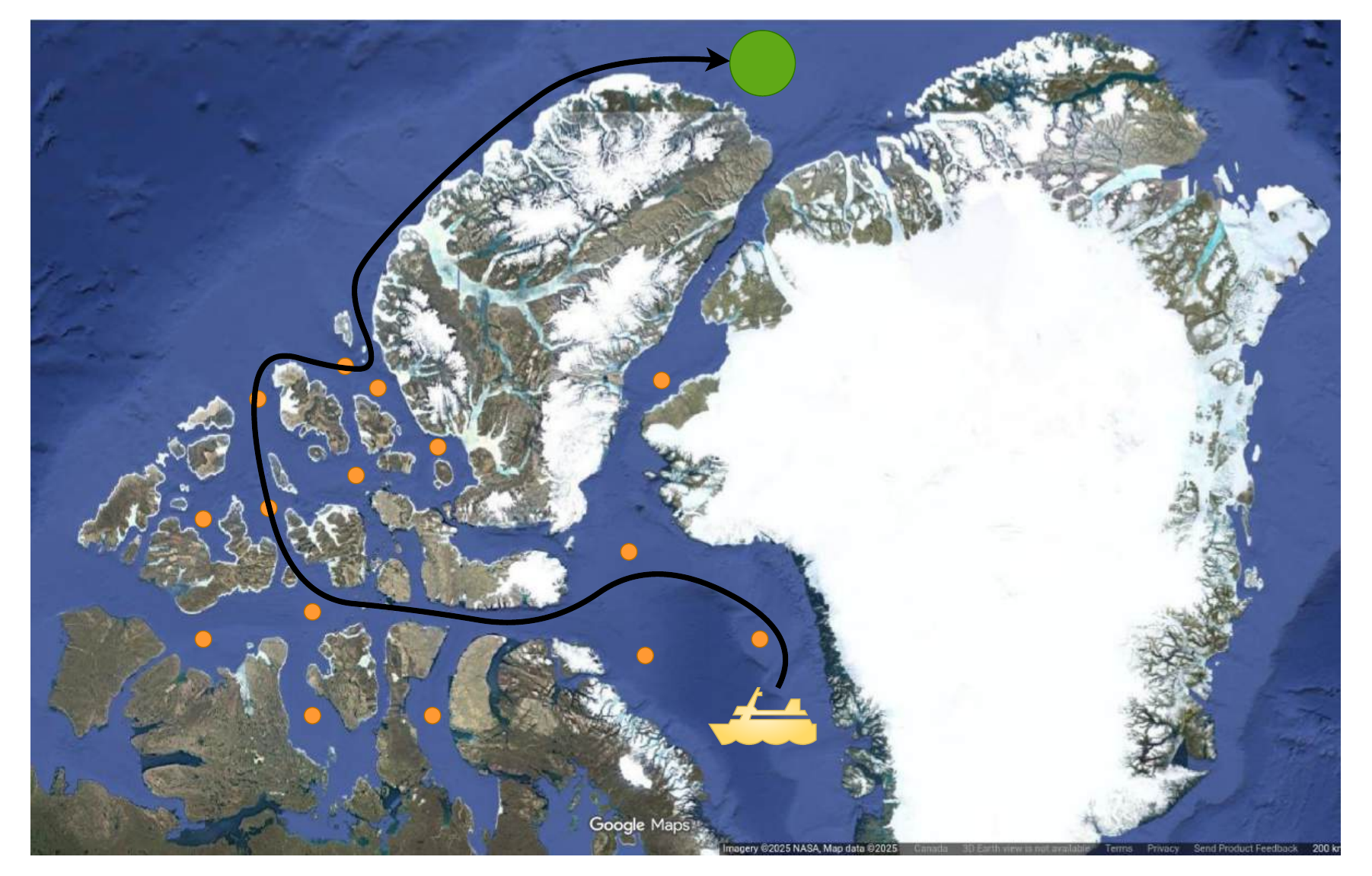} % Adjust width
    \caption{Illustration of the \gls{ipp} setting: a ship (yellow) navigates in the Arctic to collect measurements at target locations (orange), while planning a trajectory (black) to reach the goal region (green) under a strict travel budget. }
    \label{fig:motivation_fig}
\end{figure}

The \gls{ipp} problem generalizes the Orienteering Problem~\cite{golden1987orienteering} and is therefore NP-hard. Graph-based methods solve IPP by discretizing the environment into a graph and selecting informative vertices under budget constraints~\cite{binney2012branch,dutta2022informative,dutta2025informative}.  These approaches  provide strong global structure but require the definition of a pre-selected measurement set that does not capture continuous sensing along trajectories. Continuous-space methods instead optimize parameterized trajectories~\cite{marchant2014bayesian,hitz2017adaptive,popovic2020informative}, allowing measurement locations to be chosen implicitly, but obstacle handling is often overlooked. Hierarchical approaches~\cite{meera2019obstacle} account for obstacles more explicitly, but typically rely on greedy solvers as the first step, which guarantee performance only in the case of submodular objectives. Since the objective we consider, the trace of the posterior covariance, is not submodular~\cite{jawaid2015submodularity}, these guarantees do not apply. 

% In this paper we propose a hierarchical \gls{ipp} framework that combines measurement-location optimization from continuous methods with the global structural guidance of graph planning. The graph provides global guidance, and each segment is then refined with trajectory optimization under budget and obstacle constraints.

\subsection{Contributions}
We propose a hierarchical IPP framework with three stages: 
(i) a graph-based planner that computes a global budget-feasible informative path 
(ii) segment-wise budget allocation guided by geometric reachability and kernel-influence bounds, and 
(iii) spline-based refinement with hard obstacle and length constraints while pruning irrelevant obstacles. 
We establish geometric and information bounds that justify the pruning rule and motivate a smooth max-coverage relaxation for budget allocation. Through experiments on synthetic cluttered environments and Arctic datasets, our approach consistently reduces posterior uncertainty compared to graph-only and continuous baselines, achieving up to $9{\times}$ speedup over gradient-based solvers and $20{\times}$ over black-box optimizers for the continuous space methods.

\subsection{Related Work}
\gls{ipp} lies at the intersection of sensing, planning, and optimization. It builds on sensor placement \cite{krause2008near}, which studies selecting informative measurements under resource constraints, and the orienteering problem \cite{golden1987orienteering}, which addresses reward-maximizing routing under a travel budget. Together, these foundations define \gls{ipp} as active information gathering through motion. Prior work could broadly be divided into graph-based formulations and continuous-space trajectory planners.

\subsubsection{Graph-based \gls{ipp}}
Global solvers achieve optimality but scale poorly: \cite{binney2012branch} uses branch-and-bound with variance monotonicity, while \cite{dutta2023unified,dutta2022informative,dutta2025informative} introduce \gls{miqp}-based formulations for sensing/scheduling and \gls{ipp}, with a Gaussian-process generalization that minimizes the trace of the posterior covariance. To address scalability, \cite{ott2024approximate} proposes sequential dynamic programming with convex-relaxation bounds for near-optimal solutions on larger graphs. Temporal-logic specifications have also been embedded into graph-based informative planning, enabling distributed multi-robot strategies \cite{leahy2015distributed}, and formal guarantees \cite{leahy2017informative}. While effective on graphs, these methods assume measurements at vertices and do not capture continuous sensing along paths.

\subsubsection{Continuous-space \gls{ipp}}
The limitations of discretization motivate continuous methods. Sampling-based planners include RIG \cite{hollinger2014sampling}, which extends RRT* \cite{karaman2011sampling} to submodular and time-varying objectives with branch-and-bound pruning. \cite{suh2016efficient} further biases sampling toward informative regions via the cross-entropy method, improving convergence efficiency. These approaches emphasize informative sampling but do not optimize a parameterized trajectory jointly with sensing locations.
Trajectory-parameterized methods optimize the path directly. \cite{marchant2014bayesian} employs spline trajectories with Bayesian optimization, evaluating informativeness by \emph{line integrals along the curve} (continuous sensing), but without an explicit path-length budget. \cite{hitz2017adaptive} applies \gls{cmaes} \cite{hansen2006cma} with budget-aware adaptive replanning, approximating information with finite samples along the path. Convergence can be improved via informed initialization \cite{popovic2020informative}. These approaches rely on black-box optimization and are often computationally expensive and sensitive to initialization. Obstacles are typically handled with soft penalties, which can result in infeasible solutions in cluttered domains~\cite{meera2019obstacle}. Our formulation enforces hard constraints. Recent work considers distributed navigation for estimating spatiotemporal fields \cite{nguyen2025spatially}. A related problem is ergodic exploration, which seeks trajectories whose time-averaged visitation matches a prescribed information density \cite{miller2013trajectory}; unlike \gls{ipp}, informativeness there is tied to coverage of a known density rather than uncertainty reduction in a latent field.

\section{Problem Formulation and Approach}
\label{Problem}
In this section, we introduce the problem formulation in \ref{formulation} and the outline of our hierarchical approach in \ref{approach outline}. 

\subsection{Problem Formulation}
\label{formulation}
Let \( \mathcal{M} \subset \mathbb{R}^d \) be a compact convex domain representing the environment, and let \( \mathcal{M}_{\text{obs}} \subset \mathcal{M} \) denote the region occupied by static obstacles. We denote the collision-free space by \( \mathcal{M}_{\text{free}} := \mathcal{M} \setminus \mathcal{M}_{\text{obs}} \). We focus on the setting where an autonomous robot navigates \( \mathcal{M}_{\text{free}} \) to gather information about an unknown scalar function \( f: \mathcal{M} \to \mathbb{R} \) modeled as a \gls{gp}~\cite{williams2006gaussian}

\[
f \sim \mathcal{GP}(0, k),
\]
with zero mean and a kernel \( k: \mathcal{M} \times \mathcal{M} \to \mathbb{R} \). We use the squared exponential (SE) kernel:
\begin{equation}
k(x, x') = \sigma^2 \exp\left( -\frac{1}{2\ell^2} \|x - x'\|^2 \right),
\end{equation}
where \( \sigma^2 \) is the signal variance, \( \ell \) is the length-scale parameter, and $x,x' \in \mathcal{M}$ are two input locations.
Let \( T = \{\tau_1, \dots, \tau_m\} \subset \mathcal{M} \) be a finite set of test locations at which we wish to minimize the posterior uncertainty about \( f \). The robot collects noisy measurements along its piecewise continuously differentiable trajectory \( \gamma: [0,1] \to \mathcal{M}_{\text{free}} \), with arc-length
\begin{equation}
    \operatorname{Length}(\gamma):=\int_0^1\|\gamma'(t)\| \, dt.
\end{equation}
Observations are sampled at a finite set of points along \( \gamma \), denoted by \( X(\gamma) = \{\gamma(t_1), \dots, \gamma(t_n)\} \subset \text{Im}(\gamma) \), where the time parameters $t_i$ are sampled uniformly across $[0, 1]$, modeling a constant frequency measurement process.  Each measurement is 
\begin{equation}
y_i = f(\gamma(t_i)) + \epsilon_i, \quad \epsilon_i \sim \mathcal{N}(0, \sigma^2_\epsilon).
\end{equation}

We assume the noise terms $\epsilon_i$ are independent and identically distributed (i.i.d.) Gaussian random variables with zero mean and variance $\sigma_\epsilon^2$.
The posterior covariance of the test set $T$ conditioned on the observation set $X(\gamma)$ is given by

\begin{equation}
\label{eq:conditional cov}
\bar{K}_{TT}^{X(\gamma)} = K_{TT} - K_{TX(\gamma)}(K_{X(\gamma)X(\gamma)} + \sigma^2_\epsilon I)^{-1} K_{X(\gamma)T},
\end{equation}

where 
\begin{equation}
\label{eq:kernel_matrices}
\begin{aligned}
K_{TT} &= \big( k(\tau_i,\tau_j) \big)_{i,j=1}^m \in \mathbb{R}^{m\times m}, \\
K_{TX(\gamma)} &= \big( k(\tau_i,\gamma(t_j)) \big)_{\substack{i=1,\dots,m \\ j=1,\dots,n}} \in \mathbb{R}^{m\times n}, \\
K_{X(\gamma)X(\gamma)} &= \big( k(\gamma(t_i),\gamma(t_j)) \big)_{i,j=1}^n \in \mathbb{R}^{n\times n}.
\end{aligned}
\end{equation}

Our goal is to minimize the estimation error over the test set, which is equivalent to the trace of posterior covariance. 
\begin{equation}
J(\gamma) := \operatorname{Tr}( \bar{K}_{TT}^{X(\gamma)}),
\end{equation}

The IPP problem is stated as follows: 

\begin{problem}[Continuous Space \gls{ipp}]
\label{prob:continuous-ipp}
Given a start point \( s \in \mathcal{M}_{\text{free}} \), goal point \( g \in \mathcal{M}_{\text{free}} \), trajectory length scalar budget \( B > 0 \), and test set \( T \subset \mathcal{M} \), find a trajectory \( \gamma: [0,1] \to \mathcal{M}_{\text{free}} \) that is a solution to

\begin{equation}
\label{eq:continuous-ipp}
\begin{aligned}
\min_{\gamma} \quad & J(\gamma) \\
\text{subject to} \quad & \gamma(0) = s, \quad \gamma(1) = g, \\
& \operatorname{Length}(\gamma) \leq B.
\end{aligned}
\end{equation}

\end{problem}

\begin{figure}
    \centering
    \includegraphics[width=0.6\linewidth]{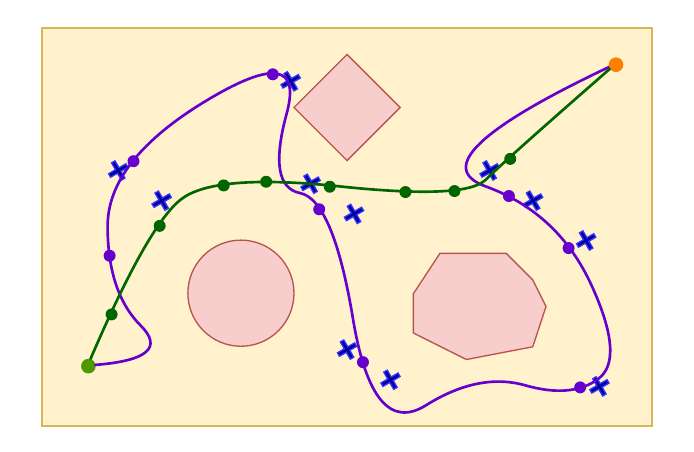}
\caption{
An instance of informative path planning in \( \mathcal{M}_{\text{free}} \). The robot moves from start \( s \) (green circle) to goal \( g \) (orange circle) along two trajectories \( \gamma \): a higher-budget (purple) and a shorter-budget (dark green) trajectory, each collecting \( n = 8 \) measurements (dots). Test points \( T \) (blue crosses) are locations of interest, and obstacles (light red) must be avoided. The longer path is able to take measurements closer to test points, while the shorter path is limited by its budget.
}
\label{fig:ipp_problem}     % had to put the label after \caption to make the fig reference work
\vspace{-10pt}
\end{figure}

Fig.~\ref{fig:ipp_problem} illustrates the problem presented in this section.

\subsection{Hierarchical Approach Outline}
\label{approach outline}
\label{Hierarchical}
\begin{figure*}
    \centering
    \includegraphics[width=0.9\linewidth]{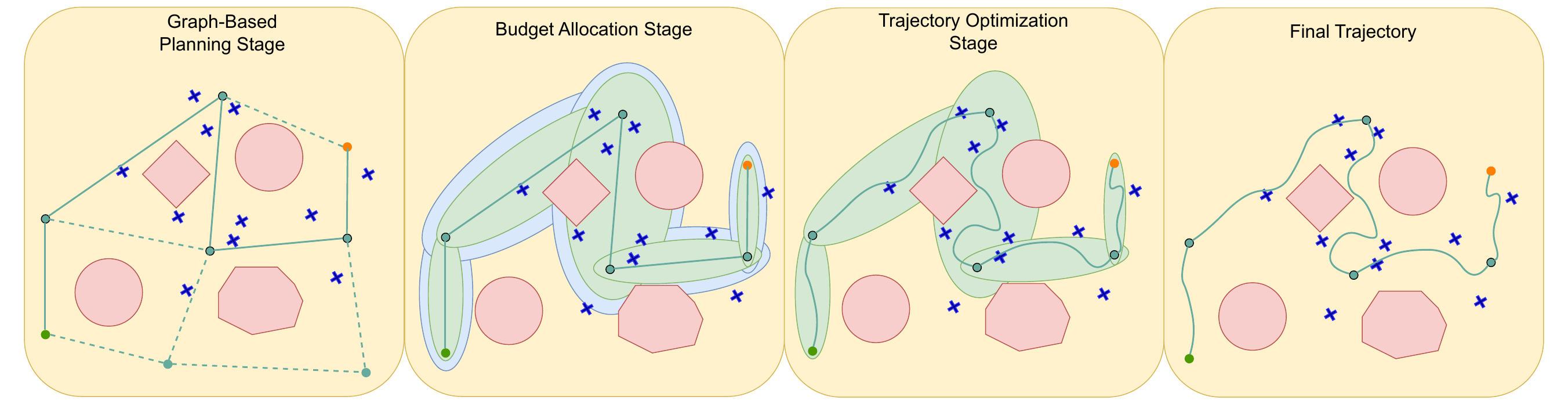}
    \caption{Illustration of the hierarchical IPP framework. 
A collision-free graph path provides global guidance. 
The path budget is then allocated across segments by maximizing coverage of test points using geometric and kernel-influence bounds. 
Finally, each segment is refined into a smooth spline trajectory that enforces obstacle and length constraints while pruning irrelevant obstacles, resulting in a continuous collision-free informative trajectory.}
    \label{fig:hierarchical-planner}
\end{figure*}

Solving Problem~\ref{prob:continuous-ipp} directly with trajectory optimization is computationally expensive and highly sensitive to initialization, since obstacle-avoidance constraints introduce nonconvex feasible regions. Our key insight is that this challenge can be mitigated by decomposing the problem into three stages (Fig.~\ref{fig:hierarchical-planner}): (i) graph-based planning for global guidance, (ii) segment-wise budget allocation that prioritizes informative segments, and (iii) spline-based trajectory optimization to refine each segment. This decomposition leverages the complementary strengths of discrete and continuous planning by first using a graph to provide global guidance, then refining individual edges through trajectory optimization, and ultimately improving computational tractability.

\subsubsection{Graph-based stage}
We first compute a budget-feasible path on a collision-free graph. This provides a fast and tractable way to capture global structure and avoid the complexity of continuous optimization and obstacle handling. The resulting path is globally optimal on the graph but assumes measurements are taken only at the vertices. In later stages, this graph path will be refined using trajectory optimization.

\subsubsection{Budget allocation stage}
A graph path specifies which edges to traverse, but not how much of the total budget should be spent refining each segment. As a key contribution, we propose a heuristic allocation method that distributes the global budget across edges to guide subsequent refinement.  We approximate the information available in the region surrounding an edge using geometric bounds and kernel-influence bounds to identify which test points could be affected by refining this segment.

\subsubsection{Trajectory refinement stage}
Finally, each segment is refined into a smooth trajectory using gradient-based spline optimization of the posterior covariance objective. The motivation for this stage is to capture the benefits of trajectory-based sensing, where information is gathered along the path rather than only at graph vertices, while reducing susceptibility to poor local minima by anchoring refinements to the graph structure. To improve tractability, we again exploit the geometric bounds results to prune obstacles that are not within a segment’s reach, thereby reducing constraint complexity.

\section{Graph-Based Planning Stage}
\label{sec:graph_planning}

We adopt the graph-based formulation of \cite{dutta2025informative}, where the graph 
\(G=(V,E)\) is given, either by discretizing the environment (e.g., grid or lattice), by 
sampling-based methods such as a probabilistic roadmap, or by manual construction. The 
vertices \(V\) are collision-free locations, and each edge \(e=(u,v)\in E\) is a feasible 
line segment with length 
\begin{equation}
\operatorname{Length}(e)=\|v-u\|.  
\end{equation}

A path \(\mathcal{P}=(v_0,\dots,v_k)\) from start \(s\) to goal \(g\) is feasible if 
\((v_{i-1},v_i)\in E\) for all \(i \in \{1, \ldots k\}\), with \(v_0=s\), \(v_k=g\), and total length within 
the budget:
\begin{equation}
    \sum_{e \in E_\mathcal{P}} \operatorname{Length}(e) \leq B,
\end{equation}
where \(E_\mathcal{P}\) is the set of edges in the path. We denote the set of all feasible 
paths by \(\mathcal{F}\). The optimal path is then
\begin{equation}
    \mathcal{P}^\star = \arg\min_{\mathcal{P} \in \mathcal{F}} 
    \operatorname{Tr}\!\big(\bar{K}_{TT}^{\mathcal{P}}\big),
\end{equation}
where \(\bar{K}_{TT}^{\mathcal{P}}\) is the posterior covariance from 
\eqref{eq:conditional cov}, but using the path vertices as measurement locations.
We rely on the convex reformulation of \cite{dutta2025informative}, which provides 
globally optimal solutions to this graph problem and can be solved efficiently using 
Gurobi~\cite{gurobi}. This stage yields a fast and tractable global plan, but with the 
simplifying assumption that measurements are only taken at vertices. In the next stages, we 
relax this assumption by allocating a budget to edges and refining them with continuous 
trajectories.

%----------------------------------------------------------------------------------
\section{Budget Allocation Stage}
In this section, we introduce a novel heuristic method for distributing the global refinement budget \(B\) across the segments of the graph-based path \(\mathcal{P}\). The key idea is to assign larger budgets to segments expected to reduce uncertainty more effectively. Since measurements are gathered along refined trajectories, we associate each segment with an informative region: the set of test points it can influence, given a certain budget. These regions are defined using geometric constructions in \ref{gemoetric bounds},  and kernel hyperparameters in \ref{informative bounds}. We then formulate an allocation problem that distributes \(B\) to maximize coverage of the resulting regions in \ref{max coverage}.

\subsection{Geometric Region Bounding}
\label{gemoetric bounds}
We first establish geometric bounds on the region reachable by any length-constrained curve between two endpoints. These results, based on arc length inequalities and elliptical containment, provide the foundation for approximating the informative region of each segment. We begin with a bound relating arc length to endpoint distances.

\begin{lemma}[Arc Length Bounds via Triangle Inequality]
\label{lemma:arc_length}
Let \( \gamma: [0, 1] \to \mathbb{R}^d \) be a piecewise continuously differentiable curve.
We define the arc length from \( a \) to \( b \) as
$
\operatorname{Length}(a \to b) := \int_a^b \|\gamma'(t)\| \, dt.
$
In particular, \( \operatorname{Length}(\gamma) := \operatorname{Length}(0 \to 1) \).

Then:
\begin{enumerate}
    \item \( \|\gamma(1) - \gamma(0)\| \leq \operatorname{Length}(\gamma) \)
    \item For any \( t \in [0, 1] \), we have
    $
    \|\gamma(t) - \gamma(0)\| + \|\gamma(1) - \gamma(t)\| \leq \operatorname{Length}(\gamma)
    $
\end{enumerate}
\end{lemma}

\begin{proof}
1. By the Fundamental Theorem of Calculus, we have
\[
\gamma(1) - \gamma(0) = \int_0^1 \gamma'(t)\, dt.
\]
Taking the norm and applying the triangle inequality for integrals we get
\[
\|\gamma(1) - \gamma(0)\| \leq \int_0^1 \|\gamma'(t)\| \, dt = \operatorname{Length}(\gamma).
\]

2. For any \( t \in [0,1] \), we apply the same argument:
\[
\|\gamma(t) - \gamma(0)\| \leq \int_0^t \|\gamma'(s)\| \, ds = \operatorname{Length}(0 \to t),
\]
\[
\|\gamma(1) - \gamma(t)\| \leq \int_t^1 \|\gamma'(s)\| \, ds = \operatorname{Length}(t \to 1).
\]
Summing both we arrive at
\[
\|\gamma(t) - \gamma(0)\| + \|\gamma(1) - \gamma(t)\| \leq \operatorname{Length}(\gamma).
\]
\end{proof}

To formalize the region reachable by length-constrained curves, we define the elliptical region characterized by two points and a length bound (i.e., foci points and major axis length).
\begin{definition}[Elliptical Region]
Given two points \( u, v \in \mathbb{R}^d \) and a length bound \( L > 0 \), define the elliptical region
\[
\mathcal{E}(u, v, L) := \left\{ x \in \mathbb{R}^d \,\middle|\, \|x - u\| + \|x - v\| \leq L \right\}.
\]
This is the set of all points whose sum of distances to \( u \) and \( v \) is no greater than \( L \). In two dimensions, this corresponds to an ellipse with foci \( u \) and \( v \) and major axis length \( L \).
\end{definition}
\begin{figure}
    \centering
    \includegraphics[width=0.9\linewidth]{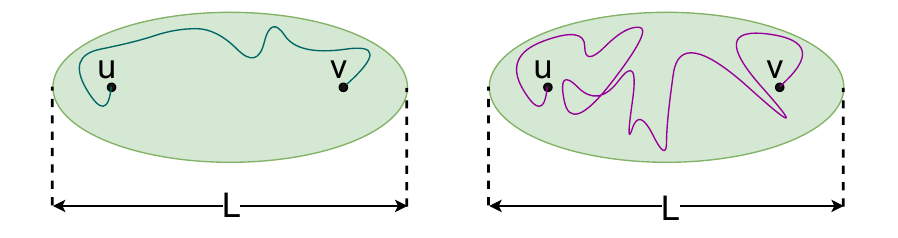}
   \caption{ \textit{Left:} Any trajectory starting at point $u$ and ending at $v$ with total length at most $L$ is fully contained within the ellipse $\mathcal{E}(u, v, L)$, defined by foci $u$, $v$ and major axis length $L$. \textit{Right:} The converse does not hold, a trajectory that remains entirely within the ellipse can still exceed length $L$ if it contains loops.}
    \label{fig:elliptical_containment}
\end{figure}
The following result implies that any length-constrained trajectory between two points lies entirely within an ellipse.
\begin{lemma}[Elliptical Containment of Length-Bounded Trajectories]
\label{lemma:containment}
Let \( u, v \in \mathbb{R}^d \), and let \( \gamma: [0,1] \to \mathbb{R}^d \) be a curve such that \( \gamma(0) = u \), \( \gamma(1) = v \), and \( \operatorname{Length}(\gamma) \leq L \). Then for all \( t \in [0,1] \), we have
$
\gamma(t) \in \mathcal{E}(u, v, L).
$
\end{lemma}

\begin{proof}
From Lemma~\ref{lemma:arc_length}, for any \( t \in [0,1] \),
\[
\|\gamma(t) - u\| + \|\gamma(t) - v\| \leq \operatorname{Length}(\gamma) \leq L.
\]
Hence, \( \gamma(t) \in \mathcal{E}(u, v, L)\) for all \( t \in [0,1] \).
\end{proof}

\begin{remark}
The converse does not hold: a curve can lie entirely within \(\mathcal{E}(u, v, L)\) but still have \( \operatorname{Length}(\gamma) > L \), e.g., if it loops inside the ellipse as depicted in Fig.~\ref{fig:elliptical_containment}
\end{remark}

\begin{figure}
    \centering
    \includegraphics[width=0.9\linewidth]{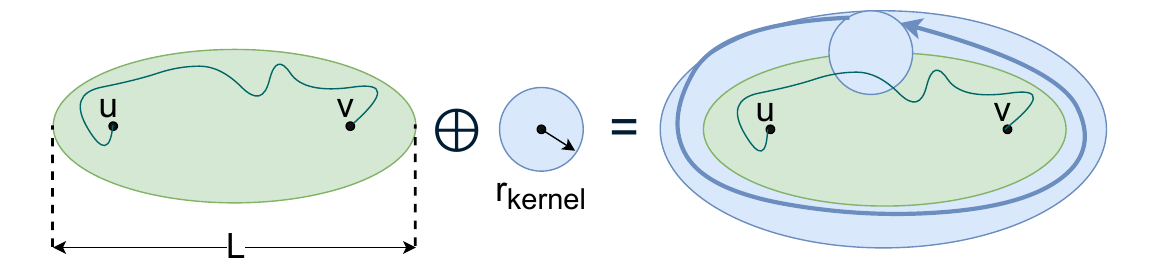}
    \caption{The Minkowski sum of the ellipse $\mathcal{E}(u, v, L)$ and a ball of radius $r_{\mathrm{kernel}}$. The resulting region represents all points within distance $r_{\mathrm{kernel}}$ from some point on the ellipse.}
    \label{fig:enter-label}
\end{figure}

\subsection{Informative Region Bounding}
\label{informative bounds}
Building on the geometric bounds, we incorporate kernel decay to characterize where measurements along a segment have non-negligible influence. This is expressed using Minkowski sums, which expand each elliptical region by the kernel’s effective radius.

\begin{definition}[Minkowski Sum]
Given two subsets \( A, B \subset \mathbb{R}^d \), the \emph{Minkowski sum} of \( A \) and \( B \) is defined as
\[
A \oplus B := \{ a + b \mid a \in A,\, b \in B \}.
\]
\end{definition}

We now use this operation to express a bound on the region where measurements have a non-negligible impact on test points. 

\begin{lemma}[Kernel Influence Radius via Minkowski Sum]
\label{lemma:kernel_influence}
Let \( k(x, x') = \sigma^2 \exp\left( - \frac{\|x - x'\|^2}{2\ell^2} \right) \) be the squared exponential kernel. Fix a tolerance \( \epsilon \in (0, \sigma^2) \), and define the kernel influence radius:
\[
r_{\text{kernel}} := \ell \sqrt{2 \log\left( \frac{\sigma^2}{\epsilon} \right)}.
\]
Let \( \gamma \) be a curve satisfying \( \operatorname{Length}(\gamma) \leq L \), and let \( \mathcal{E}(u, v, L) \) be the ellipse with foci \( u = \gamma(0) \), \( v = \gamma(1) \), and major axis length \( L \). Then, for any point \( x \notin \mathcal{E}(u, v, L) \oplus B(0, r_{\text{kernel}}) \), we have:
\[
k(x, \gamma(t)) < \epsilon \quad \text{for all } t \in [0,1].
\]
\end{lemma}

\begin{proof}
From Lemma 2, we know \( \gamma(t) \in \mathcal{E}(u, v, L) \) for all \( t \in [0,1] \). If \( x \notin \mathcal{E}(u, v, L) \oplus B(0, r_{\text{kernel}}) \), then
\[
\|x - \gamma(t)\| > r_{\text{kernel}} \quad \text{for all $t$}.
\]
Then, it follows that
\[
k(x, \gamma(t)) = \sigma^2 \exp\left( -\frac{\|x - \gamma(t)\|^2}{2\ell^2} \right) < \sigma^2 \exp\left( -\frac{r_{\text{kernel}}^2}{2\ell^2} \right).
\]
By definition of \( r_{\text{kernel}} \), we have
\[
\exp\left( -\frac{r_{\text{kernel}}^2}{2\ell^2} \right) = \frac{\epsilon}{\sigma^2},
\]
and so
\[
k(x, \gamma(t)) < \sigma^2 \cdot \frac{\epsilon}{\sigma^2} = \epsilon.
\]
\end{proof}

\begin{remark}[Negligible Effect of Distant Test Points]
Lemma~\ref{lemma:kernel_influence} shows that any test point $\tau \in T$ outside $\mathcal{E}(u,v,L)\oplus B(0,r_{\text{kernel}})$ has kernel value $<\epsilon$ with any measurement sampled on the trajectory and thus negligible effect on the \gls{gp} posterior. Budget allocation should therefore maximize the number of test points covered by the union of these regions.
\end{remark}

\subsection{Max Coverage Problem for Local Budget Allocation}
\label{max coverage}
Using Lemmas~\ref{lemma:containment} and~\ref{lemma:kernel_influence}, a segment $e=(u,v)$ with allocated length $L_e$ can only influence test points within the region $\mathcal{E}(u,v,L_e)\oplus B(0,r_{\text{kernel}})$. 
Hence, budget allocation can be cast as a maximum coverage problem, in which we distribute the total budget across edges so as to maximize the number of covered test points. 
For tractability, we approximate this influence region with the conservative ellipse 
\begin{equation}
\mathcal{R}_e=\{x \mid \|x-u\|+\|x-v\|\leq L_e+2r_{\text{kernel}}\}.  
\end{equation}
Coverage of a test point $\tau_j$ by edge $e$ is approximated with the smooth indicator  
\begin{equation}
s_{ej}(L_e)=1/\bigl(1+\exp(\alpha(d_{ej}-L_e-2r_{\text{kernel}}))\bigr),
\end{equation}  
where $d_{ej}=\|\tau_j-u\|+\|\tau_j-v\|$ and $\alpha>0$ controls sharpness. A point is considered covered if at least one edge covers it, expressed as follows 
\begin{equation}
z_j=1-\prod_{e\in E_\mathcal{P}}(1-s_{ej}(L_e)).
\end{equation}

The budget allocation problem is then formulated as the following \gls{nlp}
\begin{align}
\label{eq:alloc-obj}
\underset{\{L_e\}}{\text{maximize}} \quad & \sum_{j=1}^m  z_j \\
\label{eq:edge-lower}
\text{subject to} \quad & L_e \ge \|u-v\|, \quad \forall e \in E_\mathcal{P}, \\
\label{eq:budget-total}
& \sum_{e \in E_\mathcal{P}} L_e \le B.
\end{align}
Constraint~\eqref{eq:edge-lower} enforces that the allocation on each edge is at least its 
geometric length (i.e., the minimum required to traverse the edge), while 
Constraint~\eqref{eq:budget-total} bounds the total distributed budget by the global budget 
\(B\). We solve this \gls{nlp} in CasADi~\cite{Andersson2018} with IPOPT~\cite{wachter2006implementation}, and use the optimal $\{L_e\}$ as local budgets for trajectory refinement.

\begin{remark}
Since coverage increases with $L_e$, allocating more budget to a segment can 
only expand the set of test points it influences. The allocation problem thus balances how much additional trajectory length to assign to each segment. These budgets then guide the refinement stage, where trajectories are optimized locally within their 
assigned budget allocations.
\end{remark}

\begin{comment}
    
\begin{proposition}[Monotonicity of the Proxy Coverage Objective]
Fix all $\{L_{e'}\}_{e'\neq e}$ and a test set $T$.
For each $\tau_j\in T$, the proxy $z_j(L)=1-\prod_{e'}(1-s_{e'j}(L_{e'}))$ is nondecreasing in $L_e$.
\end{proposition}

\begin{proof}[Proof sketch]
Hold $\{L_{e'}\}_{e'\neq e}$ fixed. The logistic $s_{ej}(L_e)$ is monotonic nondecreasing in $L_e$, and Hence $\prod_{e'}(1-s_{e'j})$ is nonincreasing in $L_e$, so
$z_j(L)=1-\prod_{e'}(1-s_{e'j})$ is nondecreasing. Equivalently,
$\partial z_j/\partial L_e=\big(\prod_{e'\neq e}(1-s_{e'j})\big)\,\partial s_{ej}/\partial L_e \ge 0$.
\end{proof}

\begin{remark}
Monotonicity guarantees that increasing a segment’s budget cannot reduce its coverage contribution, aligning the relaxation with budgeted maximum coverage. In the hard-threshold limit of the sigmoid, the allocation reduces to budgeted maximum coverage (NP-hard). The smooth formulation preserves this structure while remaining differentiable, enabling solution via standard NLP.
\end{remark}

\end{comment}

\section{Trajectory Optimization Stage}
\label{sec:trajectory-optimization}
Given the ordered edges $E_\mathcal{P}$ of the graph path $\mathcal{P}$ and their allocated budgets $\{L_e\}_{e \in E_\mathcal{P}}$, we refine each edge $e=(u,v)$ into a smooth trajectory segment $\gamma_e:[0,1]\to\mathcal{M}_{\text{free}}$. Each segment is optimized independently to reduce GP posterior uncertainty over the test set $T$, while remaining within its local budget $L_e$ and inside the collision-free workspace.

\subsection{Curve Representation and Objective}
Each segment is parameterized as a degree-$d$ B-spline
\begin{equation}
\gamma_e(t) = \sum_{i=0}^r \Psi_i(t) P_i^e, \quad t\in[0,1],
\end{equation}
where $P^e=\{P_0^e,\dots,P_r^e\}$ are the control points and $\Psi_i(t)$ are the B-spline 
basis functions defined recursively by the Cox–de Boor formula from the chosen knot vector.  
We use clamped, uniform knots, which guarantee endpoint interpolation so that 
$\gamma_e(0)=P_0^e$ and $\gamma_e(1)=P_r^e$. To enforce $C^0$ continuity with the graph 
path, we fix
\begin{equation}
P_0^e = u, \qquad P_r^e = v,
\end{equation}
and optimize only over the interior control points $\{P_1^e,\dots,P_{r-1}^e\}$. These interior 
points are the decision variables of the refinement problem. Although we use  B-splines for their local control properties, the formulation is compatible with other curve bases such as Bézier. 

To evaluate our objective, we sample the curve at $N$ points $\{t_i\}_{i=1}^N$ to get measurement points $X_e = \{\gamma_e(t_i)\}_{i=1}^N$. The objective is the posterior covariance matrix conditioned on $X_e$:
\begin{equation}
\label{eq:traj-obj}
J_e(P^e) = \operatorname{Tr}\!\big(\bar{K}_{TT}^{X_e}\big),
\end{equation}
where $\bar{K}_{TT}^{X_e}$ is defined in \eqref{eq:conditional cov}. Minimizing $J_e(P^e)$  encourages trajectories that place samples in informative regions and reduce posterior uncertainty over $T$.

\subsection{Nonlinear Program}
The trajectory refinement problem for a segment $e$ is 
\begin{align}
\label{eq:nlp-obj}
\min_{P^e} \quad & J_e(P^e) \\[4pt]
\text{s.t.} \quad 
\label{eq:nlp-budget}
& \operatorname{Length}(\gamma_e) \leq L_e, \\[4pt]
\label{eq:nlp-env}
& \phi_{\text{env}}(\gamma_e(t_j^{\text{env}})) \ge 0, \quad j=1,\dots,N_{\text{env}}, \\[4pt]
\label{eq:nlp-obs}
& \phi_{\text{obs}}(\gamma_e(t_j^{\text{obs}})) \ge 0, \quad j=1,\dots,N_{\text{obs}}.
\end{align}

Constraint~\eqref{eq:nlp-budget} ensures that the refined trajectory for edge $e$ respects its allocated budget $L_e$. Constraint~\eqref{eq:nlp-env} enforces that the trajectory remains inside the workspace $\mathcal{M}$ by requiring positive signed distances to the environment boundary. Constraint~\eqref{eq:nlp-obs} ensures obstacle avoidance by requiring 
positive signed distances to obstacles.  
This nonlinear program is solved in CasADi~\cite{Andersson2018} with IPOPT~\cite{wachter2006implementation}.

\begin{remark}[Obstacle Pruning for Efficiency]
Obstacles disjoint from the ellipse $\mathcal{E}(u,v,L_e)$ can be ignored, since Lemma~\ref{lemma:containment} guarantees the refined trajectory remains inside this region. Pruning them reduces constraint evaluations, improves solver efficiency, and does not affect feasibility.
\end{remark}

\subsection{Relying Only on Trajectory Optimization}

Fig.~\ref{fig:complex-obstacle} shows examples of optimized B-spline trajectories in an obstacle-cluttered environment. In the left panel, the trajectory successfully navigates around obstacles while taking measurements close to test points. However, because the underlying problem is an \gls{nlp}, the solution may get trapped in a poor local minima. This is illustrated in the right panel, where a slight change in the start location results in a poor trajectory. This motivates the need to use our hierarchical framework instead of solely relying on continuous optimization methods.

\begin{figure}[h]
    \centering
    \includegraphics[width=0.5\linewidth]{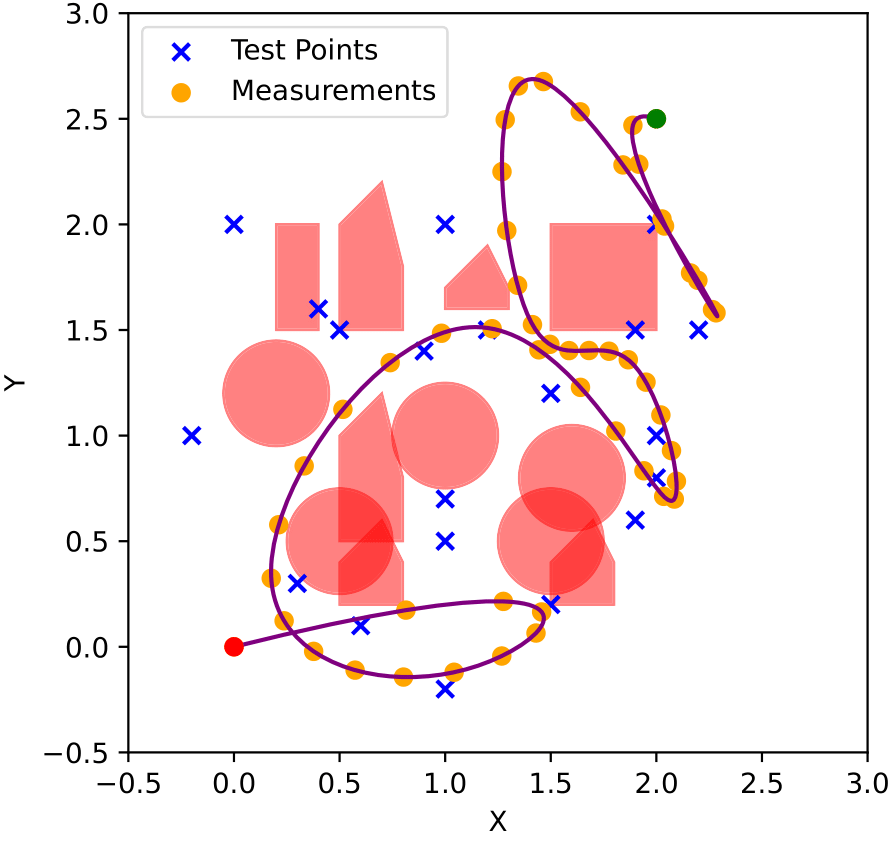}\hfill
    \includegraphics[width=0.5\linewidth]{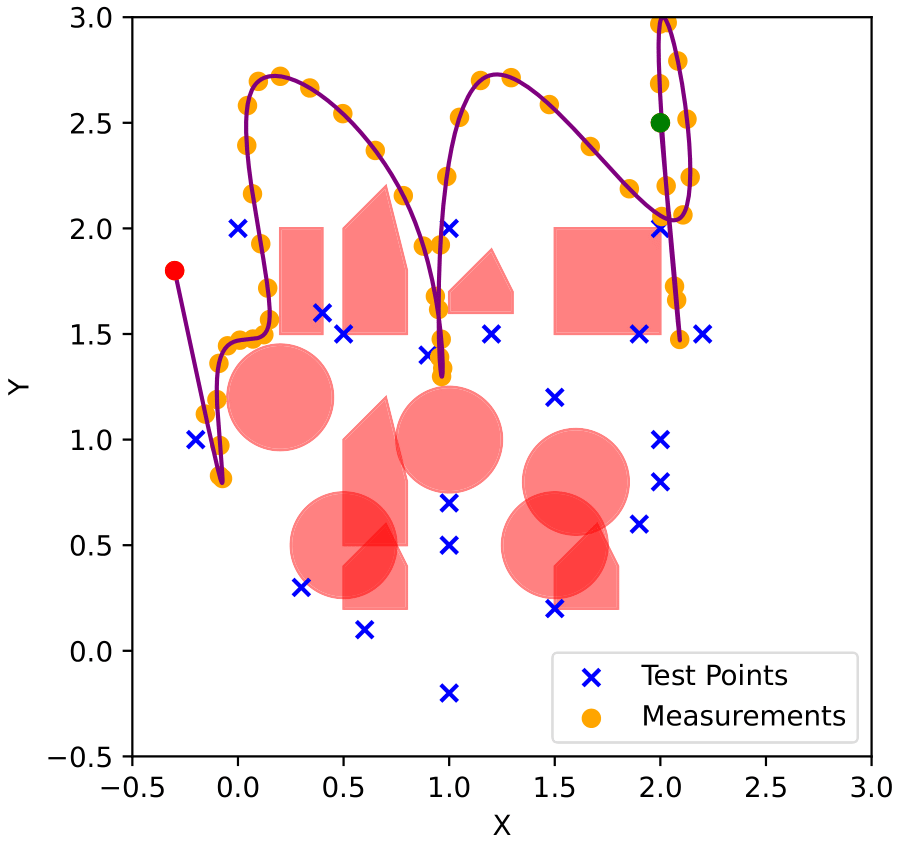}
    \caption{Examples of B-spline trajectory optimization in an obstacle-rich environment with 12 control points. 
    \textit{Left:} a successful trajectory that avoids obstacles and remains close to informative test points. 
    \textit{Right:} a suboptimal trajectory where the optimizer converges to a poor local minimum, wasting budget in low-information areas.}
    \label{fig:complex-obstacle}
    \vspace{-10pt}
\end{figure}

%===============================================

\section{Simulation Results}
\label{results}
We evaluate our approach in both synthetic and geographic map settings using the ice-cover Arctic dataset and benchmark it against continuous-space and graph-based baselines. 

\subsection{Comparisons in obstacle-cluttered environments}
In this section, we evaluate our hierarchical framework against a graph-based \gls{micp} planner~\cite{dutta2025informative}, a continuous-space solver using the non-adaptive version of CMA-ES CIPP~\cite{hitz2017adaptive}, and a CasADi-based IPOPT solver. Both CMA-ES CIPP and CasADi optimize cubic B-splines with 14 control points, while the hierarchical planner uses shorter B-spline segments with 5 points, requiring less flexibility since the global structure is fixed by the graph. 
\paragraph{Setup}
The environment (Fig.~\ref{fig:complex-obstacle}) spans $3.5 \times 3.5$ meters and contains twelve polygonal and circular obstacles. A fixed graph with 11 vertices and 22 edges is manually constructed to cover the free space. Start and goal vertices are sampled randomly with a minimum separation of $1.0$~m to avoid trivial instances and ensure trajectories must traverse obstacles. We adopt a squared exponential kernel with parameters $\ell = 0.35$~m, $\sigma^2 = 10$, and $\sigma_\epsilon^2 = 0.1$, together with a budget $B = 14$ and $m = 35$ random test points. For this kernel, correlations decay to near zero at a distance of about $3\ell \approx 1$~m. Consequently, informative trajectories are likely to navigate between obstacle regions to capture informative measurements. By contrast, if the field were highly correlated, almost any feasible trajectory would yield similar information regardless of its path, and if the field were much less correlated, single measurements would provide little information about nearby test points, leading to a useless setup. For fairness, all planners collect the same number of measurements. The hierarchical and graph planners distribute them uniformly along their selected segments, while continuous solvers place them along their optimized trajectories. 

\paragraph{Results}
The hierarchical planner achieves the lowest mean objective (17.8) with reduced variance across runs, compared to CMA-ES CIPP (25.1), CasADi-IPOPT (25.8), and the graph baseline (27.4) (Fig.~\ref{fig:complex-eval-boxplot}). This represents a $\sim29\%$ improvement over direct trajectory optimization and a $\sim35\%$ improvement over the graph-only baseline. In runtime, the graph planner remains fastest (0.27s), but the hierarchical method (10.2s) is nearly an order of magnitude faster than CasADi-IPOPT (92.0s) and twenty times faster than CMA-ES CIPP (204.2s). By restricting optimization to elliptical regions informed by the graph, the hierarchical approach reduces obstacle constraints and converges more efficiently, balancing between the speed of discrete planning and the performance of continuous refinement.

\begin{figure}[t]
    \centering

    \subfloat[Objective values across solvers.]{
        \includegraphics[width=0.65\linewidth]{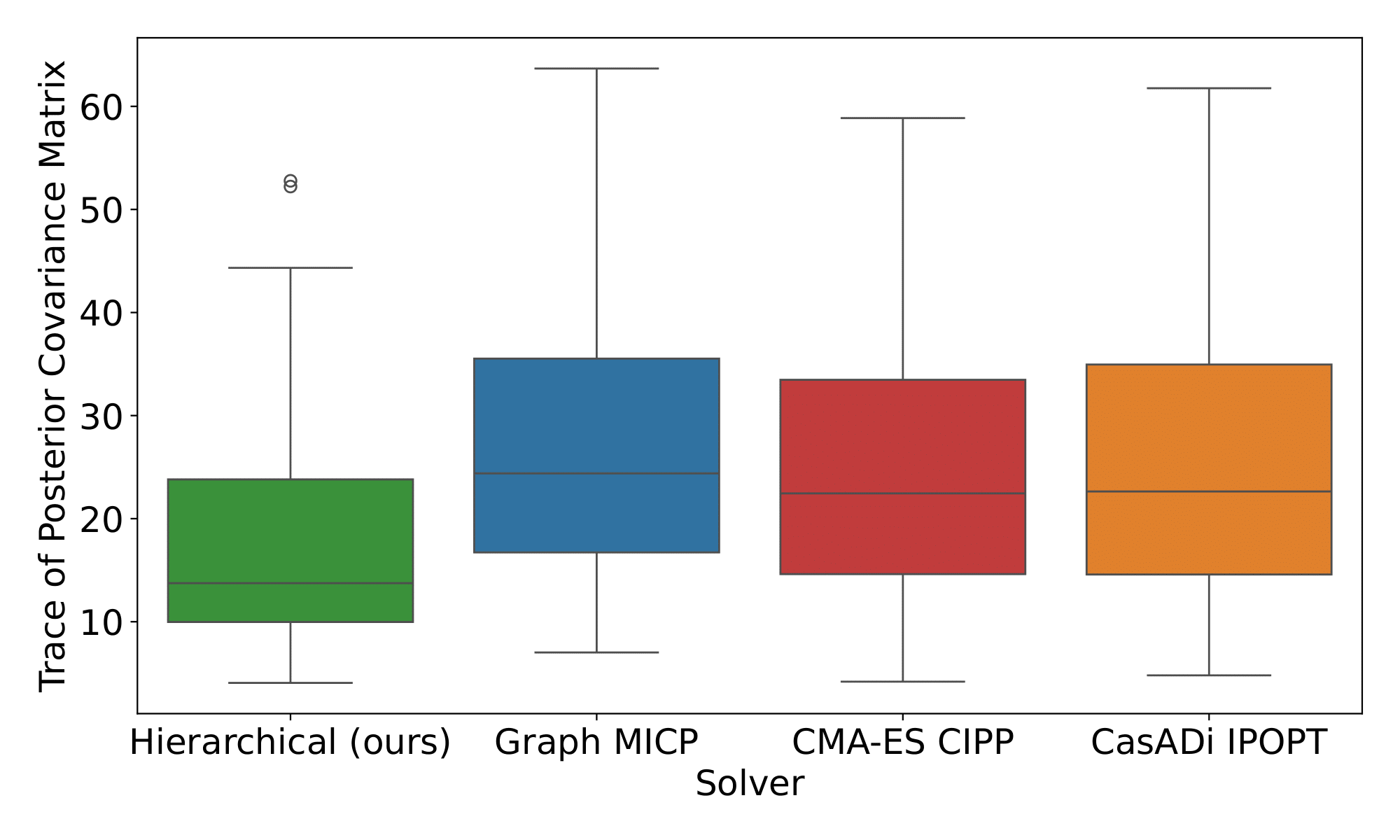}
    }\hfill
    \subfloat[Runtime across solvers.]{
        \includegraphics[width=0.65\linewidth]{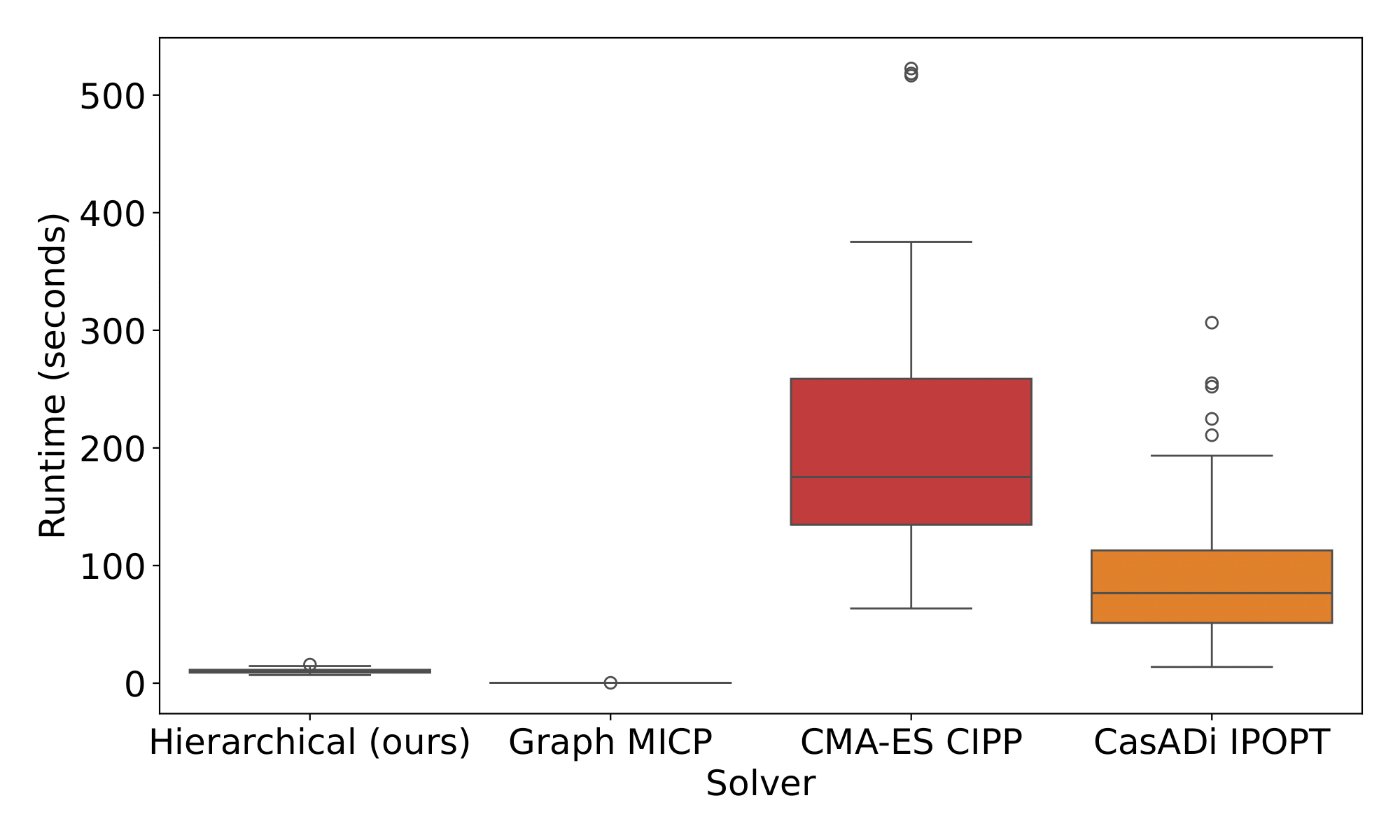}
    }

    \caption{Performance comparison in obstacle-cluttered environments. Hierarchical planner attains the lowest objective (30–35\%) better than baselines, and runs $9\times$ faster than CasADi-IPOPT and $20\times$ faster than CMA-ES.
}
    \label{fig:complex-eval-boxplot}
    \vspace{-10pt}
\end{figure}

\subsection{Arctic Dataset}
\paragraph{Setup 2} We evaluate our method on sea-ice cover in the Canadian Arctic Archipelago, using ERA5 daily data (April 2025) from the Copernicus Climate Change Service~\cite{era5_daily_single} with land boundaries from Cartopy. For numerical stability and interpretable kernel hyperparameters, the domain is uniformly scaled by a factor of $10^5$, so that one unit in the plots corresponds to 100 km in the real world while preserving the geometry of the region. Non-convex landmasses are decomposed into 90 convex obstacles. The planning task models an autonomous surface vessel navigating between start and goal locations to collect informative sea-ice measurements along its trajectory.  Kernel hyperparameters are obtained by maximizing the marginal likelihood, yielding a squared exponential kernel with lengthscale $\ell \approx 1.8$ (about 180km in real geography), output scale $\sigma^2 \approx 0.1$, and noise variance $\sigma_\epsilon^2 \approx 10^{-2}$. Since sea-ice cover lies in $[0,1]$.
In this scenario, the user provides an importance field \( I \) as shown in the top-left panel of Fig.~\ref{fig:Archipelago}, which assigns higher values to regions of interest. This field is normalized into a probability density, from which \( m=30 \) test points are sampled, constrained to lie in free space. Due to the large number of partitioned obstacles and the cluttered environment, pure trajectory optimization methods failed. CasADi-IPOPT terminated before finding a feasible solution, and \gls{cmaes} returned a trajectory that collided with obstacles, with both being given 1000 iterations.
In contrast, our hierarchical optimization framework successfully generates feasible start–goal paths that satisfy both budget and obstacle constraints while collecting informative measurements along the way. To assess estimation quality, we compare two cases: a lower budget of 40 units (4,000 km in real geography) and a higher budget of 65 units (6,500 km). Performance is measured using the importance-weighted RMSE, which accounts for the user-defined importance field.

\paragraph{Importance-weighted RMSE}  
To evaluate approximation quality, we compute the root mean squared error weighted by a normalized importance distribution \( \tilde{I}(x) \), where \( \tilde{I}(x) = I(x) / \sum_j I(x_j) \). Letting \( \mu(x) \) denote the posterior mean of the \gls{gp} and \( f(x) \) the ground truth, the metric is approximated over a grid \( \{x_j\}_{j=1}^n \) as:
\begin{equation}
\text{Weighted RMSE} \approx \sqrt{ \sum_{j=1}^{n} \tilde{I}(x_j) \left( \mu(x_j) - f(x_j) \right)^2 }.
\label{eq:importance_rmse_discrete}
\end{equation}

\paragraph{Results}  
Table~\ref{tab:importance_metrics} summarizes the estimation performance across different path budgets. As expected, the higher-budget trajectory attains a lower importance-weighted RMSE than the lower-budget one, reflecting its ability to place measurements more effectively in regions prioritized by the user. Fig.~\ref{fig:Archipelago} illustrates the sea-ice data and the corresponding trajectories generated by the hierarchical planner, while Fig.~\ref{fig:error_maps} shows the spatial distribution of the weighted mean squared error across the map.

\begin{table}[t]
    \centering
    \caption{Estimation performance under different path budgets.}
    \label{tab:importance_metrics}

    % tighten padding & row spacing
    \setlength{\tabcolsep}{4pt}
    \renewcommand{\arraystretch}{1.15}

    % spread columns across full width
    \begin{tabular*}{\linewidth}{@{\extracolsep{\fill}}lccc@{}}
        \hline
        \textbf{Method} &
        \shortstack{\textbf{RMSE}\\\textbf{(at test pts)}} &
        \shortstack{\textbf{Weighted}\\\textbf{RMSE}\textbf{(in the whole map)}}  \\
        \hline
        Hierarchical planner (low budget) & 0.314 & 0.521\\
        Hierarchical planner (high budget)& 0.156 & 0.381 \\
        \hline
    \end{tabular*}
\end{table}

\begin{figure}[t]
    \centering

    \subfloat[Importance field $I$]{%
        \includegraphics[width=0.48\linewidth]{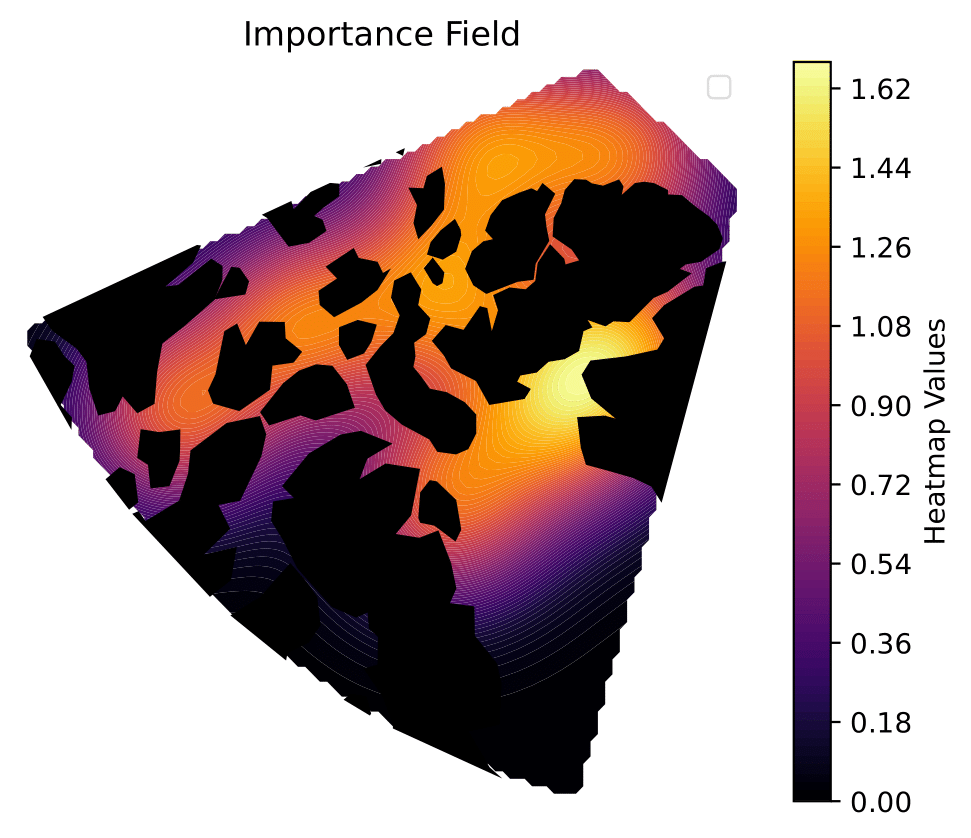}%
    }\hfil
    \subfloat[Real ice cover]{%
        \includegraphics[width=0.48\linewidth]{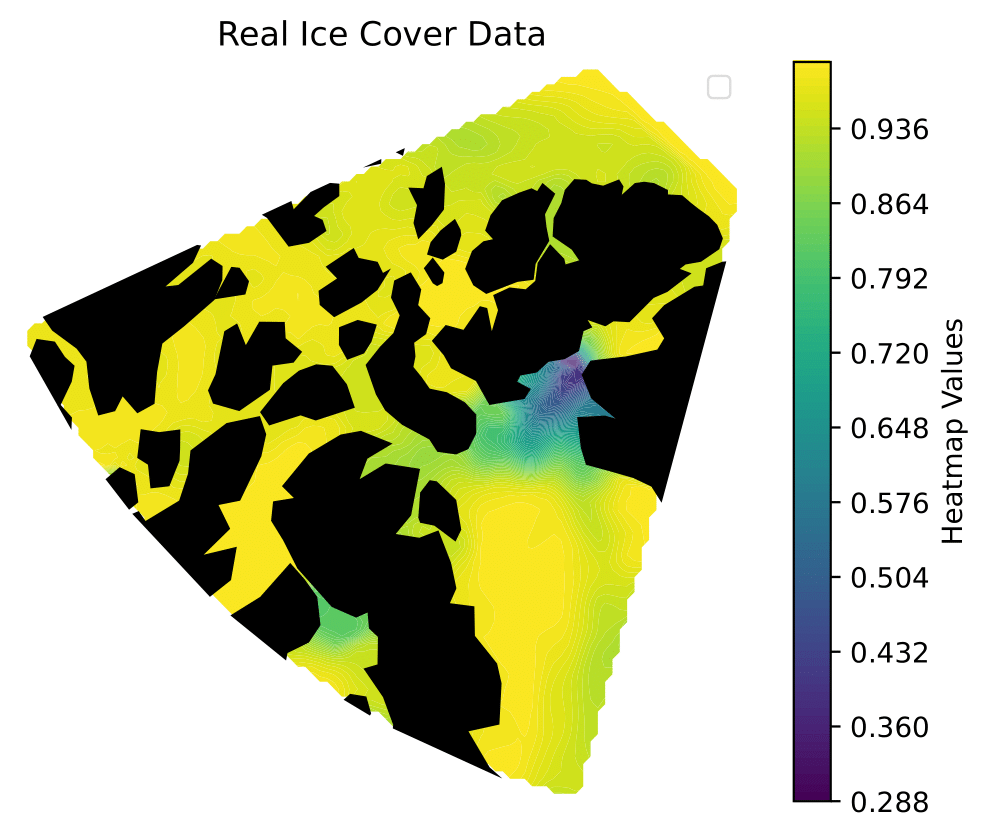}%
    }

    \subfloat[Low-budget plan]{%
        \includegraphics[width=0.48\linewidth]{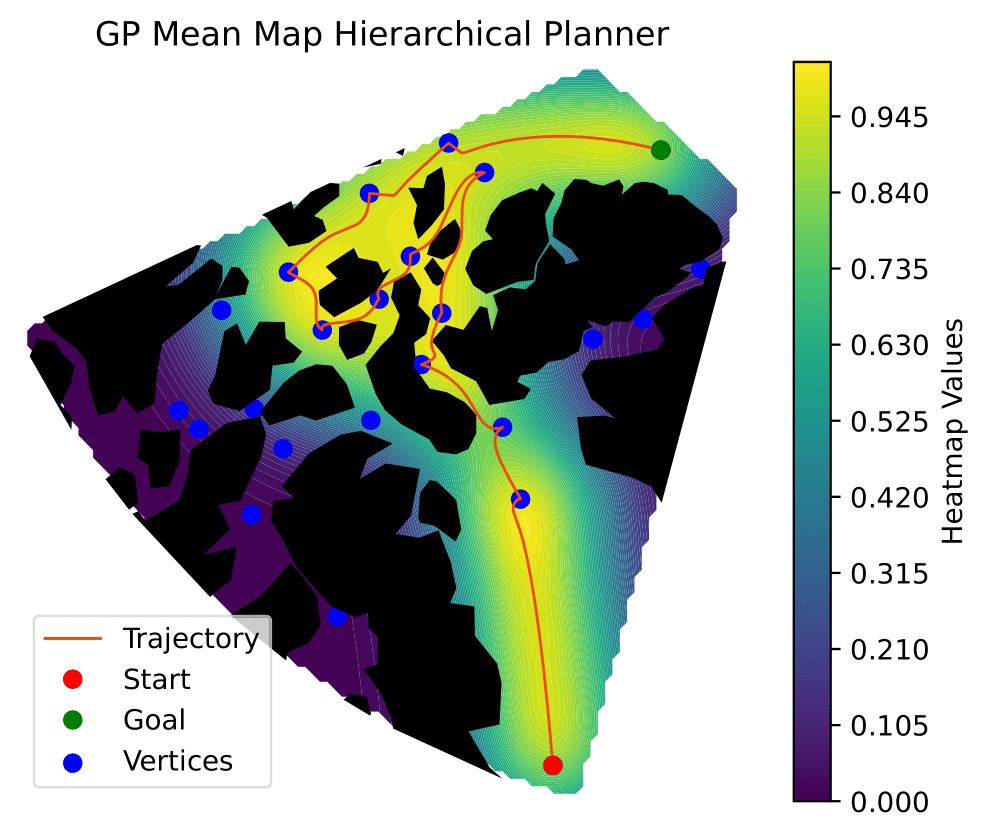}
    }\hfil
    \subfloat[High-budget plan]{%
        \includegraphics[width=0.48\linewidth]{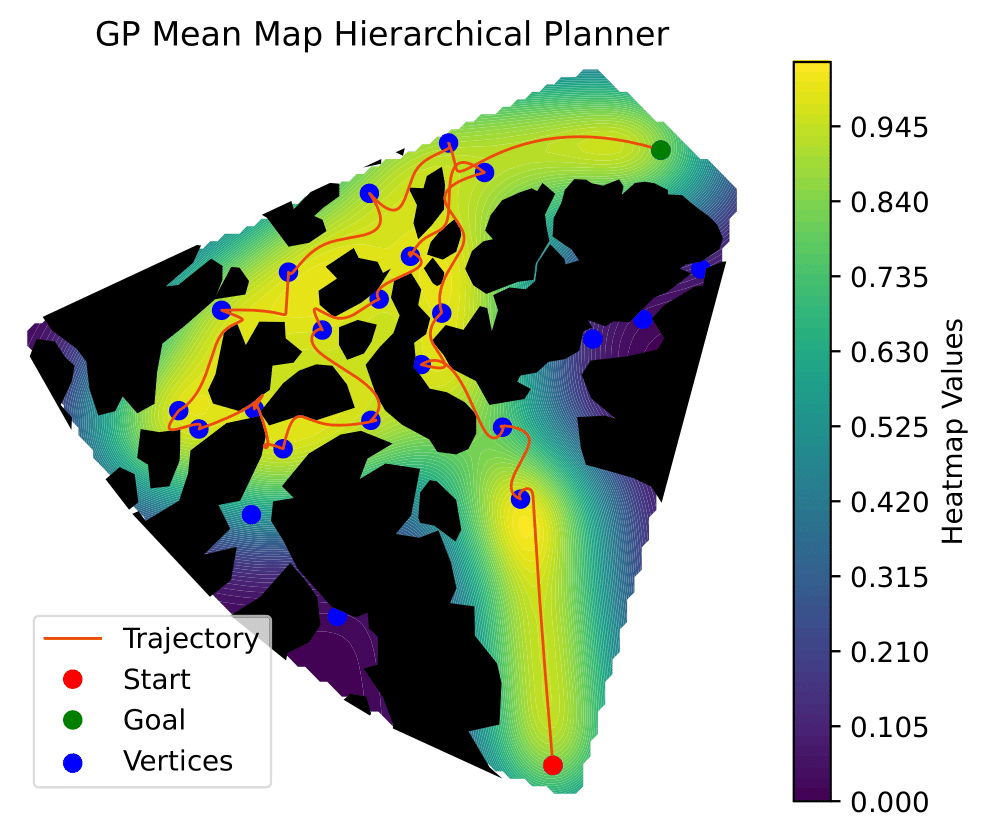}%
    }

    \caption{Hierarchical IPP in the Canadian Arctic Archipelago. Satellite data from \cite{era5_daily_single}. Plans are overlaid on the GP posterior mean; obstacles in black.}
    \label{fig:Archipelago}
    \vspace{-10pt}
\end{figure}

\begin{figure}[t]
    \centering

    \subfloat[Mean squared error map for a low-budget trajectory.]{
        \includegraphics[width=0.45\linewidth]{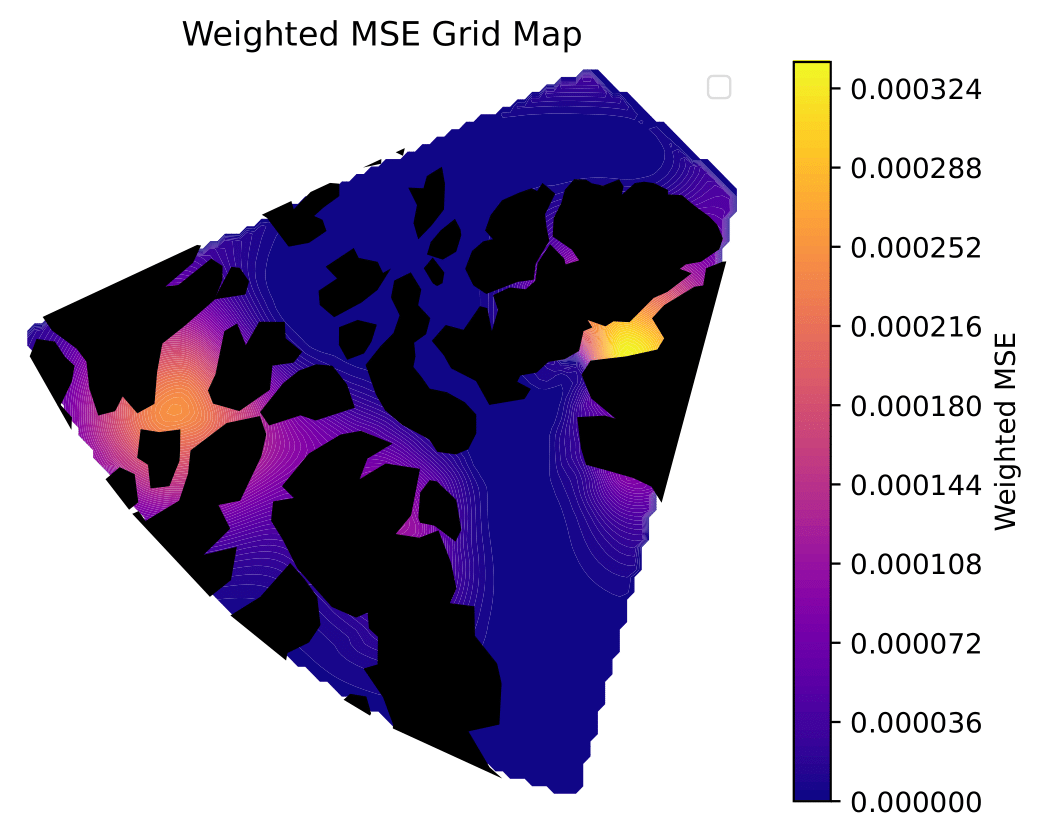}
    }\hfill
    \subfloat[Mean squared error map for a higher-budget trajectory.]{
        \includegraphics[width=0.45\linewidth]{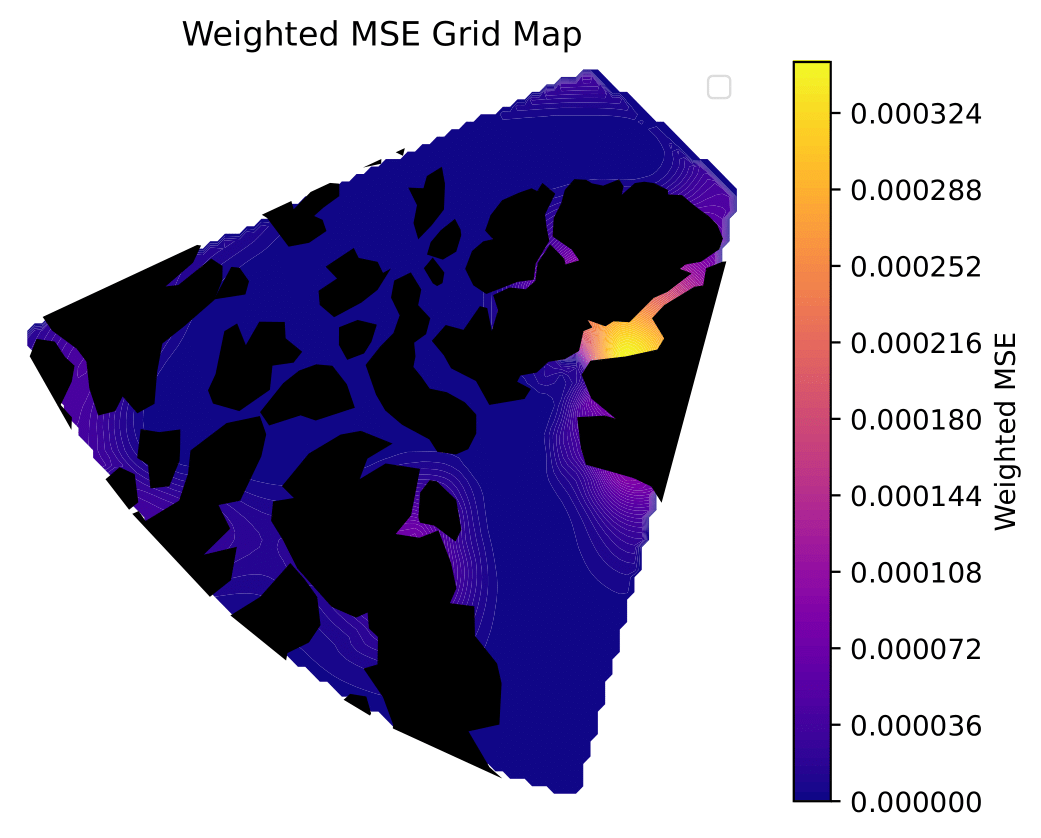}
    }

    \caption{Posterior mean squared error maps under two different path budgets. A higher budget allows for more informative measurements, reducing estimation error. The error scale is small (on the order of \(10^{-4}\)) due to the limited range of the measured quantity, as ice cover values lie between 0 and 1.}
    \label{fig:error_maps}
    \vspace{-10pt}
\end{figure}

\section{Conclusion and Future Work}
This work introduced a hierarchical \gls{ipp} framework that decomposes planning into graph search, budget allocation, and trajectory optimization. Simulations showed that it improves estimation error over purely combinatorial or continuous methods, while also achieving faster runtimes than continuous solvers. Experiments on synthetic and geographic maps with Arctic datasets further demonstrated performance improvements over baselines and faster runtime than continuous space methods.

Future work includes developing adaptive variants that update the \gls{gp} online as new data are collected, particularly when hyperparameters are partially known. We also plan to incorporate stochastic budget constraints, since real traversal costs are often uncertain~\cite{quann2017energy}. Finally, we aim to leverage recent advances in graphs of convex sets~\cite{marcucci2024shortest} to move beyond pre-defined graphs and unify discrete planning and continuous optimization in a single stage.
\label{Conclusion}

%\begin{IEEEkeywords}

%\end{IEEEkeywords}

\bibliographystyle{IEEEtran}
\bibliography{refs} % no .bib extension

\vspace{12pt}
\color{red}

\end{document}